\documentclass{article}

\usepackage[preprint, nonatbib]{neurips_2023}

\usepackage[utf8]{inputenc} 
\usepackage[T1]{fontenc}    
\usepackage{hyperref}       
\usepackage{url}            
\usepackage{booktabs}       
\usepackage{amsfonts}       
\usepackage{nicefrac}       
\usepackage{microtype}      
\usepackage{xcolor}         

\usepackage{amsmath}
\usepackage{amsthm}
\usepackage{amssymb}
\usepackage{wrapfig}
\usepackage{enumitem}
\usepackage{tikz}
\usepackage{booktabs}
\usepackage{float}
\usepackage{url}
\usepackage{subcaption}

\theoremstyle{plain}

\newtheorem{manualtheoreminner}{Theorem}
\newenvironment{manualtheorem}[1]{%
  \renewcommand\themanualtheoreminner{#1}%
  \manualtheoreminner
}{\endmanualtheoreminner}

\newtheorem{manualpropinner}{Proposition}
\newenvironment{manualprop}[1]{%
  \renewcommand\themanualpropinner{#1}%
  \manualpropinner
}{\endmanualpropinner}

\newtheorem{theorem}{Theorem}[section]
\newtheorem{proposition}[theorem]{Proposition}

\theoremstyle{definition}
\newtheorem{definition}[theorem]{Definition}

\theoremstyle{remark}

\title{Higher Order Gauge Equivariant CNNs on Riemannian Manifolds and Applications}

\author{%
  Gianfranco Cort\'{e}s\textsuperscript{\textdagger}\thanks{Correspondence to: \texttt{gcortes@ufl.edu}} 
  \quad
   \textbf{Yue Yu}\thanks{Department of CISE, University of Florida, Gainesville, FL}
   \quad
   \textbf{Robin Chen}\thanks{Department of Biomedical Engineering, University of Florida, Gainesville, FL}
   \\
   \textbf{Melissa Armstrong}\thanks{Department of Neurology, University of Florida, Gainesville, FL} 
   \quad
   \textbf{David Vaillancourt}\thanks{Department of Applied Physiology and Kinesiology, University of Florida, Gainesville, FL} 
   \quad
   \textbf{Baba C. Vemuri}\textsuperscript{\textdagger} \\
}

\begin{document}

\maketitle
\begin{abstract}
With the advent of group equivariant convolutions in deep networks literature, spherical CNNs with $\mathsf{SO}(3)$-equivariant layers have been developed to cope with data that are samples of signals on the sphere $S^2$.
One can \textit{implicitly} obtain $\mathsf{SO}(3)$-equivariant convolutions on $S^2$ with significant efficiency gains by \textit{explicitly} requiring gauge equivariance w.r.t. $\mathsf{SO}(2)$. In this paper, we build on this fact by introducing a higher order generalization of the gauge equivariant convolution, whose implementation is dubbed a \textit{gauge equivariant Volterra network} (GEVNet). This allows us to model spatially extended nonlinear interactions within a given receptive field while still maintaining equivariance to global isometries. We prove theoretical results regarding the equivariance and construction of higher order gauge equivariant convolutions. Then, we empirically demonstrate the parameter efficiency of our model, first on computer vision benchmark data (e.g. spherical MNIST), and then in combination with a convolutional kernel network (CKN) on neuroimaging data. In the neuroimaging data experiments, the resulting two-part architecture (CKN + GEVNet) is used to automatically discriminate between patients with Lewy Body Disease (DLB),  Alzheimer's Disease (AD) and Parkinson's Disease (PD) from diffusion magnetic resonance images (dMRI). The GEVNet extracts micro-architectural features within each voxel, while the CKN extracts macro-architectural features across voxels. This compound architecture is uniquely poised to exploit the intra- and inter-voxel information contained in the dMRI data, leading to improved performance over the classification results obtained from either of the individual components.

\end{abstract}

\section{Introduction}

\subsection{Related Literature in Computer Vision} 

Many applications in computer vision call for the analysis of signals sampled on an underlying manifold that is non-Euclidean. A ubiquitous example of such a non-Euclidean space is the 2-sphere $S^2$. Atmospheric signals such as temperature and wind fields, measurements of the cosmic microwave background, and omnidirectional images captured by fish-eye lenses are all examples of signals whose domain ought to be modeled by $S^2$. It is desirable for CNNs that process spherical signals to be equivariant to rotations, i.e. equivariant to the action of the Lie group $\mathsf{SO}(3)$. This is analogous to the translation equivariance satisfied by planar CNNs.

The need to generalize CNNs to homogeneous spaces such as $S^2$ and even arbitrary Riemannian manifolds has since been answered by a flurry of techniques arising from an emerging subfield known as geometric deep learning (GDL) \cite{5gs}. Here we review but a fraction of the relevant GDL literature. 

Previously published CNNs that are equivariant to 2D and 3D rotations were reported in \cite{Worrall16, WeilerCesa, bekkers_roto-translation_2018} and \cite{Cohen-ICML17, Cohen-ICLR18, esteves2017polar,CGNet}, respectively. More recently, work aimed at generalizing beyond the plane and sphere was reported in \cite{Kondor-icml18,BanerjeeISBI19,CohenNIPS19, Bekkers2020B-Spline}, in which their networks can in theory cope with arbitrary homogeneous spaces. From here, group convolutions on homogeneous spaces were pushed further along two diverging research directions of relevance, which this paper sets out to merge together. In one direction, the higher order analogue (Volterra expansion) of convolution on homogeneous spaces was introduced in \cite{Banerjee-TPAMI20}, in which the parameter efficiency of higher order convolutions was demonstrated. In the other direction, the condition of having a homogeneous base space was relaxed, giving rise to the gauge equivariant convolutions presented in \cite{Cohen-Weiler-Kicanaoglu-Welling2019, dehaan2021, Weiler2021} which are valid on arbitrary Riemannian manifolds. It is still valuable to consider gauge equivariant convolutions on homogeneous spaces, as they afford high spatial resolutions while maintaining feature maps of low bandwidth (yielding better computational/memory efficiency). This is in stark contrast to group convolutional techniques (\`a la Fourier), where the spatial resolution and feature map bandwidth are invariably coupled together.

\subsection{Diffusion MRI and Neurology Application}


\paragraph{Diffusion MRI:} Diffusion weighted magnetic resonance  imaging (dMRI) is a non-invasive imaging technique that provides a way to probe the axonal fiber connectivity in the body by making the magnetic resonance (MR) signal sensitive to water diffusion through the tissue being imaged \cite{basser1994mr}. Typically, diffusion sensitizing magnetic field gradients are applied along a large number of directions and the response MR signal is collected at {\it each voxel} along these
directions. Thus, for each direction, the data contains an entire 3D MR volume. This amounts to the existence of a function $f: S^2 \times \mathbb{R}^{+} \longrightarrow \mathbb{R}$ at each voxel, assigning an intensity value to the voxel for a given direction and magnitude (of the applied magnetic field gradients). In our work, we fix the magnitude (known as a single-shell model), and this reduces us to the existence of a function $f:S^2 \longrightarrow \mathbb{R}$, or a scalar field on $S^2$. It follows that any end-to-end network processing dMRI data ought to contain a module that performs spherical convolutions within voxels if we are to respect the geometry of the data.
\
\paragraph{Neurology Application:} Dementia with Lewy Bodies (DLB), Alzheimer's Disease (AD) and Parkinson's Disease (PD)  are common forms of neurodegenerative disorders. DLB is 
the second most common dementia in the U.S., but 1 in 3 cases may
be missed and individuals with DLB are frequently misdiagnosed in early stages, most commonly as AD or PD dementia \cite{ThomasAJ2017}. Therefore, neurologists face a challenging diagnostic task and may be assisted with an automated classifier trained on dMRI data to discriminate between DLB, AD and PD groups. To the best of our knowledge, there are no end-to-end networks reported in literature for this specific neurology application.


There are several approaches to feature extraction from dMRI data that have had moderate to good success in  the classification of neurodegenerative disorders. Here we simply cite a few representative methods that include the use of scalar-valued indices and more sophisticated morphometric indices  \cite{Prodoehl2013,Ozarslan05,pasternakMRM09,NOODI-Alexander,BanerjeeCVPR2016}.
All of these features however are "hand-crafted" and not learned from the data. 
In \ref{medical_expt}, we propose a novel end-to-end network composed of a \textit{convolutional kernel network} with quadratic kernels (CKN2) and a \textit{gauge equivariant Volterra network} (GEVNet) on $S^2$ to be trained on dMRIs. 

{\bf Our Contributions:} The key contributions of this work are: (1) a higher order generalization of the first order gauge equivariant convolution introduced in \cite{Cohen-Weiler-Kicanaoglu-Welling2019}, (2) proofs regarding the equivariance and construction of higher order gauge equivariant convolutions, (3) the resulting implementation on $S^2$, dubbed a GEVNet, and (4) experimental results of the GEVNet on spherical benchmark tasks and a tandem CKN + GEVNet application for the classification of neurodegenerative disorders.


The rest of the paper is organized as follows: In Section \ref{background}, we present material on gauge equivariant and kernel convolutions, concepts that are heavily used throughout the paper. Section \ref{methods} contains our key theoretical contribution, namely, the higher order gauge equivariant convolution. 
This is followed by a remark on the relation between CKNs with polynomial kernels and higher order Volterra expansions. 
Section \ref{arch} presents some implementation details. Section \ref{expts} contains experimental results on benchmark computer vision datasets and a neuroimaging application. Finally, we draw conclusions in Section \ref{conc}. \vspace{2pt}


\section{Background} \label{background}
In this section, we briefly review the key definitions and results underlying the theory of gauge equivariant CNNs, as presented in \cite{Cohen-Weiler-Kicanaoglu-Welling2019} and \cite{Weiler2021}. Then, we summarize the construction of convolutional kernel networks presented in \cite{Mairal2016}. 
\subsection{First Order Gauge Equivariant Convolutions}

\begin{wrapfigure}{r}{0.45\textwidth}
\vspace{-10pt}
\centering
\includegraphics[width=0.4\textwidth]{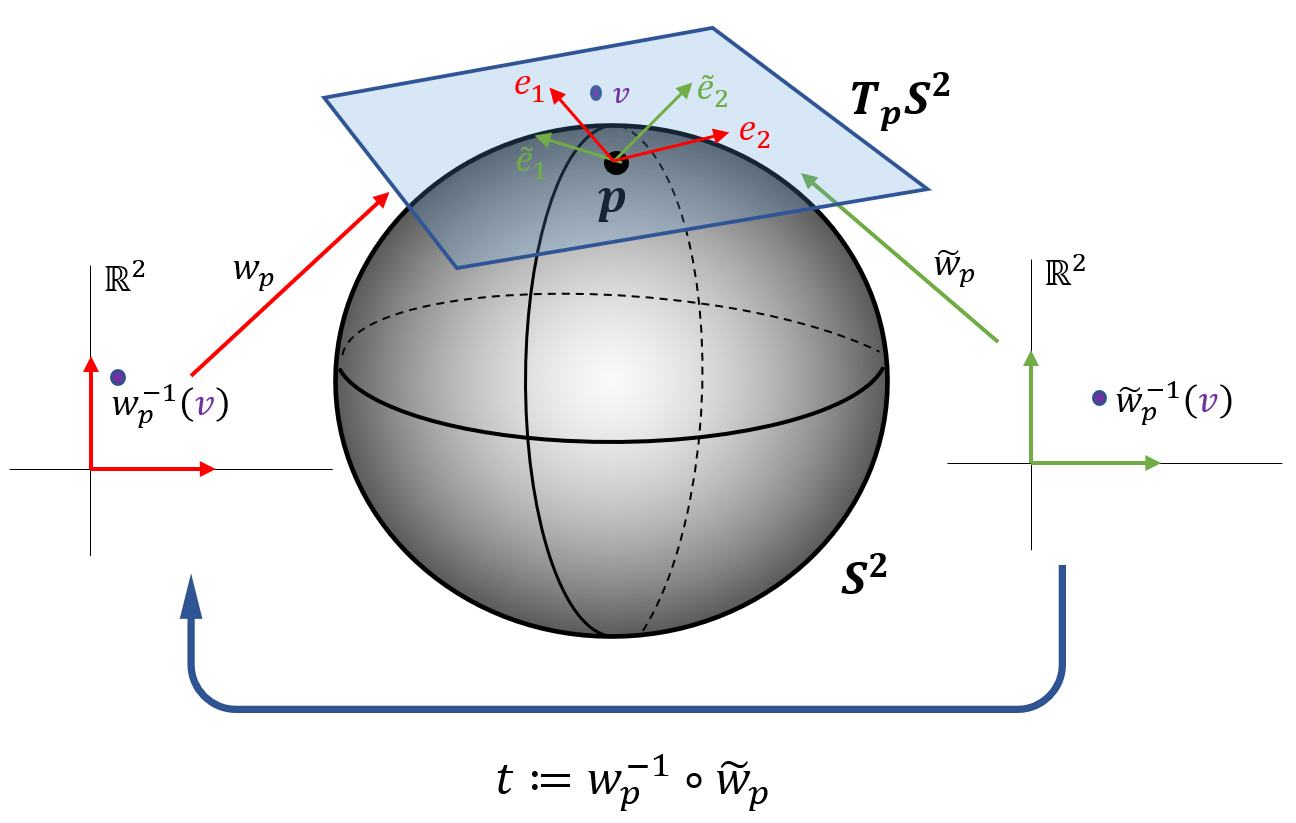}
\caption{Illustration of two gauges (red and green) and a gauge transformation between them (blue) }
\vspace{-10pt}
\label{gauge}
\end{wrapfigure}
Let $M$ be a connected, geodesically complete Riemannian manifold of dimension $2$. Recall that to each $p\in M$, we can associate the tangent space $T_pM$ and its dual $T_p^*M$ \cite{Lee1997}. A feature map $f$ on $M$ is modeled as a smooth tensor field on $M$, i.e. $f(p)$ is a tensor for each $p \in M$. We denote this as $f(p)\in T_pM^{\otimes r} \otimes T_p^*M^{\otimes s} =: \mathfrak{T} M\vert_p$, where $\otimes$ denotes a tensor product and the superscript, $^{\otimes r}$, denotes an $r$-fold tensor product. 
A convolution will then map a tensor field $f_{\mathrm{in}}$ on $M$ to a tensor field $f_{\mathrm{out}}$ on $M$. The reason for working with tensor-valued features will be made explicit below.

A tensor field $f$ on $M$ is a geometrically intrinsic construction that in general will lack a canonical global coordinatization. Therefore, to perform computations, we must locally assign a frame to each tangent space $T_pM$, which then induces a frame on the corresponding tensor space $\mathfrak{T}M\vert_p$. Such an assignment is called a \textit{gauge}, which is to be thought of as a (local) smoothly parameterized collection of linear isomorphisms $w_p:\mathbb{R}^2 \longrightarrow T_pM$. Given a frame $\{b_1, b_2\}$ on $\mathbb{R}^2$, $w_p$ induces a frame on $T_pM$ given by $\mathbf{e}_1 = w_p(b_1)$ and $\mathbf{e}_2 = w_p(b_2)$. Figure \ref{gauge} illustrates this concept.

Since the choice of gauge is arbitrary, we are forced to consider the transition between gauges via a \textit{gauge transformation}. At a point $p \in M$ with pointwise gauges $w_p:\mathbb{R}^2\longrightarrow T_pM$ and $\widetilde{w}_p:\mathbb{R}^2 \longrightarrow T_pM$, this amounts to considering the map $t := w_p^{-1} \circ \widetilde{w}_p$. The map $t$ acts naturally on frames and tangent vector coefficients, where in computations we use a corresponding change-of-basis matrix. 

For example, if our feature map $f$ is a vector field, then $f(p) \in T_pM$. A tangent vector $f(p) = \mathbf{v}$ is an intrinsic geometric object, but we can coordinatize $\mathbf{v}$ with respect to the pointwise gauges $w_p$ and $\widetilde{w}_p$ above, yielding coordinates $w_p^{-1}(\mathbf{v}) = a = (a_1, a_2)$ w.r.t. $\{\mathbf{e}_1, \mathbf{e}_2\}$ and $\widetilde{w}_p^{-1}(\mathbf{v}) = \widetilde{a} = (\widetilde{a}_1, \widetilde{a}_2)$ w.r.t. $\{\widetilde{\mathbf{e}}_1, \widetilde{\mathbf{e}}_2\}$. Under the assumption that $\widetilde{w}_p = w_p \circ t$, it must be the case that $\widetilde{a} = t^{-1}(a)$, so that $\mathbf{v} = \widetilde{w}_p(\widetilde{a}) = (w_p \circ t)(\widetilde{a}) = (w_p \circ t \circ t^{-1})(a) = w_p(a)$ is coordinate independent. The relation $\widetilde{a} = t^{-1}(a)$ is referred to as a tensor \textit{transformation law}, which describes how the coordinatization of a tensor changes w.r.t. a change of frame.

More generally, higher order tensors on a $2$-manifold have transformation laws of the form $\widetilde{a} = \rho(t^{-1})(a)$, where $\rho:G \longrightarrow \mathsf{GL}(2, \mathbb{R})$ is a group representation (i.e. $\rho(t_1t_2) = \rho(t_1)\rho(t_2)$ for all $t_1, t_2 \in G$) and $G$ is a Lie subgroup of $\mathsf{GL}(2, \mathbb{R})$ called the \textit{structure group}. As $\rho$ encodes the tensor transformation law, we also refer to $\rho$ as the \textit{tensor type}. In this work, we are only interested in orthonormal frames with positive orientation, so we will only be concerned with $G = \mathsf{SO}(2)$ from here on out. 

Since our feature maps are no longer scalar-valued and the base manifold is no longer Euclidean, there exist several differences between the classical convolution and the gauge equivariant convolution. Firstly, 
we must specify beforehand the tensor types $\rho_{\mathrm{in}}$ and $\rho_{\mathrm{out}}$ corresponding to the input and output of a convolution.
Secondly, the kernel $K$ is upgraded to a smooth matrix-valued map $\mathbb{R}^2 \longrightarrow \mathbb{R}^{d_{\mathrm{out}} \times d_{\mathrm{in}}}$ with compact support, where $d_{\mathrm{in}}$ and $d_{\mathrm{out}}$ are the dimensions of the tensor spaces $\mathfrak{T}_{\mathrm{in}}M\vert_p$ and $\mathfrak{T}_{\mathrm{out}}M\vert_p$ in which the input and output features lie, respectively.
Thirdly, we must parallel transport the features in a given patch to a common tensor space so that operations such as feature addition become meaningful.

With these differences in mind, let $f_{\mathrm{in}}$ be a feature map of type $\rho_{\mathrm{in}}$ and $K: \mathbb{R}^2 \longrightarrow \mathbb{R}^{d_{\mathrm{out}} \times d_{\mathrm{in}}}$ a kernel as above. We remind the reader that the Riemannian exponential map at a point $p \in M$ is a map $\operatorname{exp}_p : T_pM \longrightarrow M$ taking a tangent vector $\mathbf{v} \in T_pM$ to the point $\operatorname{exp}_p(\mathbf{v}) \in M$ at which one arrives after following a geodesic with velocity $\mathbf{v}$ for one unit of time. Letting $q_v := \mathrm{exp}_p(w_pv)$, the convolved feature map $f_{\mathrm{out}} = K \star f_{\mathrm{in}}$ is given pointwise by
\begin{equation}
f_{\mathrm{out}}(p) := \int\limits_{\mathbb{R}^2}K(v)\rho_{\mathrm{in}}(t_{p \gets q_v})f_{\mathrm{in}}(q_v)\,dv,
\label{eq:1}
\end{equation}
where $t_{p \gets q_v}$ denotes the $\mathsf{SO}(2)$-valued gauge transformation taking the frame on $q_v$ (after parallel transport to $p$) to the frame on $p$. Note that the action of $\rho_{\mathrm{in}}(t_{p \gets q_v})$ on $f_{\mathrm{in}}(q_v)$ subsumes the familiar translational shift $f(t - \tau)$ seen in the classical Euclidean convolution. In \cite{Cohen-Weiler-Kicanaoglu-Welling2019}, it was shown that (\ref{eq:1}) is equivariant to a gauge transformation at $p$ if and only if $K$ is $\mathsf{SO}(2)$\textit{-steerable}, i.e. $K$ satisfies
\begin{equation}
    K(t^{-1}v) = \rho_{\mathrm{out}}(t^{-1})K(v)\rho_{\mathrm{in}}(t)
\label{eq:2}
\end{equation}
for all $t \in \mathsf{SO}(2)$ and $v \in \mathbb{R}^2$. It is crucial to note that in the case where $f_{\mathrm{in}}$ and $f_{\mathrm{out}}$ are scalar-valued feature maps, both $\rho_{\mathrm{in}}$ and $\rho_{\mathrm{out}}$ are trivial, meaning equation (\ref{eq:2}) reduces to $K(t^{-1}v) = K(v)$, i.e. we are constrained to isotropic kernels. This motivates the need to generalize to tensor-valued features so that we can detect anisotropy. We conclude with a critical result describing the relation between local equivariance w.r.t. gauge transformations and global equivariance w.r.t. isometries.
\begin{theorem} [Theorem 8.11, \cite{Weiler2021}] Convolutions (as defined in $\mathrm{(\ref{eq:1})}$) with $\mathsf{SO}(2)$-steerable kernels are equivariant w.r.t. the action of orientation preserving isometries $\phi \in \mathrm{Isom}_{+}(M)$.
\label{theorem:1}
\end{theorem}
Hence, in particular, we have that gauge equivariant convolutions where $G = \mathsf{SO}(2)$ and $M = S^2$ are equivariant to global rotations $\phi \in \mathsf{SO}(3)$.

\subsection{Convolutional Kernel Networks}
\label{section:2.2}
Convolutional kernel networks (CKNs) \cite{Mairal2016} are the culmination of merging traditional CNNs with kernel methods. In the CKN setting, a feature map at the $\ell^{\mathrm{th}}$ layer is modeled as a function $I_{\ell}:\Omega_{\ell}\longrightarrow \mathcal{H}_{\ell}$, where $\Omega_{\ell}$ is some Euclidean domain and $\mathcal{H}_{\ell}$ is a reproducing kernel Hilbert space (RKHS). As in any learning scheme that wishes to exploit the kernel trick, an appropriate positive definite kernel $K_{\ell}: \mathcal{X}_{\ell - 1} \times \mathcal{X}_{\ell - 1} \longrightarrow \mathbb{R}$ must be defined. We take $\mathcal{X}_{\ell}$ to be the space of image patches (of a fixed size) with support in $\Omega_{\ell}$ and $K_{\ell}(\mathbf{x}, \mathbf{x}') :=\vert\vert \mathbf{x} \vert\vert\,\vert\vert\mathbf{x}'\vert\vert\, \kappa_{\ell}(\langle \frac{\mathbf{x}}{\vert\vert\mathbf{x}\vert\vert}, \frac{\mathbf{x}'}{\vert\vert\mathbf{x}'\vert\vert} \rangle)$, where $\kappa_{\ell}$ is a nice $\mathbb{R}$-valued function. For instance, if we are dealing with $3 \times 3$ patches of a 2D greyscale image, then $\mathcal{X} = \mathbb{R}^{3\cdot 3 \cdot 1} = \mathbb{R}^9$. Recall that a positive definite kernel $K_{\ell}$ induces an embedding $\varphi_{\ell}(\mathbf{x}) = K_{\ell}(\mathbf{x}, -)$ taking image patches $\mathbf{x} \in \mathcal{X}_{\ell - 1}$ into the RKHS $\mathcal{H}_{\ell}$. Instead of learning a set of weights, a CKN learns a finite-dimensional subspace $\mathcal{F}_{\ell} \subset \mathcal{H}_{\ell}$ for each $\ell$ such that the projection residuals (projection onto $\mathcal{F}_{\ell}$) of the embedded $\varphi_{\ell}(\mathbf{x})$ are minimized. The ability to choose $\kappa_{\ell}$ to our liking is the main property of CKNs that we will leverage in \ref{section:3.2}.

\section{Methodology} \label{methods}
We begin by generalizing the theory of first order gauge equivariant convolutions to its higher order analogue. Then, we describe how the classical Volterra series can be recast into the kernel convolution framework.

\subsection{Higher Order Gauge Equivariant Convolutions}
Let $M$ be a connected, geodesically complete Riemannian 2-manifold as before. We avoid unnecessary generalization by maintaining that $G = \mathsf{SO}(2)$. However, the following definitions and results can be modified for arbitrary structure groups $G \subseteq \mathsf{GL}(2, \mathbb{R})$.
\begin{definition} A $k^{\mathrm{th}}$ order kernel $K^{(k)}$ of type $(\rho_{\mathrm{out}}, \rho_{\mathrm{in}})$ is a smooth map 
\begin{equation}
K^{(k)}:\bigoplus_{i=1}^k\mathbb{R}^2\longrightarrow \mathbb{R}^{d_{\mathrm{out}} \times d_{\mathrm{in}}^k}.
\end{equation}
The symbol $\oplus$ denotes the direct sum operation, e.g., $\mathbb{R} \oplus \mathbb{R}=\mathbb{R}^2$.
We are treating $\mathbb{R}^{d_{\mathrm{out}} \times d_{\mathrm{in}}^k}$ as a coordinatization of the space of linear maps $\mathrm{Hom}(\mathfrak{T}_{\mathrm{in}}M\vert_p^{\otimes k}, \mathfrak{T}_{\mathrm{out}}M\vert_p)$ w.r.t. a chosen frame, where $\mathfrak{T}_{\mathrm{in}}M\vert_p$ and $\mathfrak{T}_{\mathrm{out}}M\vert_p$ are the tensor spaces in which the input and output features lie, respectively. 
\end{definition}

\begin{definition}
A $k^{\mathrm{th}}$ order kernel $K^{(k)}$ of type $(\rho_{\mathrm{out}}, \rho_{\mathrm{in}})$ is said to be $\mathsf{SO}(2)$-\textit{steerable} iff it satisfies a $k^{\mathrm{th}}$ order steerability constraint given by
\begin{equation}
K^{(k)}(t^{-1}v_1, \ldots, t^{-1}v_k) =  \rho_{\mathrm{out}}(t^{-1})K^{(k)}(v_1, \ldots, v_k) 
\rho_{\mathrm{in}}^{\otimes k}(t) 
\label{eq:3}
\end{equation}
for all $t \in \mathsf{SO}(2)$ and $v_1, \ldots, v_k \in \mathbb{R}^2$.
Here, $\rho_{\mathrm{in}}^{\otimes k}$ denotes the $k$-fold tensor product representation.
\end{definition}

\begin{definition}
Let $\mathcal{K} = \{K^{(1)}, \ldots, K^{(m)}\}$ be a collection of kernels, where $K^{(k)}$ is an $\mathsf{SO}(2)$-steerable $k^{\mathrm{th}}$ order kernel of type $(\rho_{\mathrm{out}}, \rho_{\mathrm{in}})$ for $k = 1, \ldots, m.$ Letting $q_v := \mathrm{exp}_p(w_pv)$, the $m^{\mathrm{th}}$ order expansion of a feature map $f_{\mathrm{in}}$ is given by $V^m_{\mathcal{K}}(f_{\mathrm{in}}):=\sum_{k=1}^mK^{(k)} \star f_{\mathrm{in}}$, where
\begin{equation}
(K^{(k)}\star f_{\mathrm{in}})(p) :=\underbrace{\int\limits_{\mathbb{R}^2}\mkern -7mu\cdots\mkern -7mu\int\limits_{\mathbb{R}^2}}_{k\,\,\mathrm{times}}K^{(k)}(v_1, \ldots, v_k) \\
 \left(\bigotimes_{i=1}^k\rho_{\mathrm{in}}(t_{p \gets q_{v_i}})f_{\mathrm{in}}(q_{v_i})\right)dv_1\cdots dv_k.
\end{equation}
\label{def:3.3}
\end{definition}
We are now ready to present a theorem on the gauge equivariance of the higher order operator $V_{\mathcal{K}}^m$. Such a result is necessary for two reasons: (a) we would like for the output of $V_{\mathcal{K}}^m$ to be independent of the coordinatization of its input and (b) we want $V_{\mathcal{K}}^m$ to further enjoy $\mathsf{SO}(3)$-equivariance in the case where $M = S^2$ (cf. Theorem \ref{theorem:1}).
\begin{theorem}
$V^m_{\mathcal{K}}$ is gauge equivariant. That is, if $w_p$ and $\widetilde{w}_p$ are two pointwise gauges at $p \in M$ related by an $\mathsf{SO}(2)$-valued gauge transformation $\widetilde{w}_p = w_p \circ t$, then $V^m_{\mathcal{K}}(f_{\mathrm{in}})(p)$ transforms as $\rho_{\mathrm{out}}(t^{-1})V^m_{\mathcal{K}}(f_{\mathrm{in}})(p)$.
\end{theorem}

\vspace{-1ex}
\begin{proof}
We will sketch the proof here and give the details in the appendix. 
By linearity, it suffices to show that $K^{(k)} \star (\cdot)$ is gauge equivariant for each $k = 1, \ldots, m$. Assuming an initial coordinatization w.r.t. $w_p$, the vector components will transform as $v \mapsto t^{-1}v$ and the parallel transport term will transform as $\rho_{\mathrm{in}}(t_{p \gets q_{v}})f_{\mathrm{in}}(q_{v}) \mapsto \rho_{\mathrm{in}}(t^{-1})\rho_{\mathrm{in}}(t_{p \gets q_{v}})f_{\mathrm{in}}(q_{v})$. After substituting these transformation laws and equation (\ref{eq:3}) into the definition above, the only corrective term that remains is $\rho_{\mathrm{out}}(t^{-1})$, which can then be factored out by linearity.
 \end{proof}

In any application of steerable kernels, the main difficulty arises from generating solutions to the steerability constraint \cite{WeilerCesa}, which is usually done via some combination of analytic and representation-theoretic methods. The most desirable outcome is to obtain a complete basis for the subspace of kernels that are steerable. Our use case will only involve second order expansions of vector-valued feature maps, and thus we make do with the following sufficient condition on second order $\mathsf{SO}(2)$-steerable kernels.

\begin{proposition}
Let $K^{(1)}_{01}$ be a first order $\mathsf{SO}(2)$-steerable kernel of type $(\rho_0, \rho_1)$ and $K^{(1)}_{11}$ a first order $\mathsf{SO}(2)$-steerable kernel of type $(\rho_1, \rho_1)$, where $\rho_0$ is the trivial representation (i.e. $\rho_0(t) = 1$) and $\rho_1$ is the standard representation mapping $t\in\mathsf{SO}(2)$ to the usual $2 \times 2$ rotation matrix. Then, a second order kernel of the form $K^{(2)}_{11}(v_1, v_2):= K^{(1)}_{11}(v_1) \otimes K^{(1)}_{01}(v_2)$ is $\mathsf{SO}(2)$-steerable (i.e. satisfies eq. $\mathrm{(\ref{eq:3})}$) of type $(\rho_1, \rho_1)$.
\label{prop:1}
\end{proposition}

\begin{proof}
    This follows from applying equation (\ref{eq:3}) and the mixed-product property. See the appendix in the supplementary material for more detail.
\end{proof}

\vspace{-2pt}
Using Proposition \ref{prop:1}, we can generate a set $\mathcal{B}$ of second order $\mathsf{SO}(2)$-steerable kernels from the first order solutions obtained in \cite{WeilerCesa} to map $\rho_1$ features to $\rho_1$ features. Note that \ref{prop:1} can be modified to account for the other relevant cases (e.g. $\rho_1 \to \rho_0$, etc.). We then parameterize the space of second order $\mathsf{SO}(2)$-steerable kernels w.r.t $\mathcal{B}$, learning only the scalar coefficients. 

\subsection{Higher Order Convolutional Kernel Networks}
\label{section:3.2}
The traditional convolution operator $H_1$ admits a natural generalization called the \textit{Volterra series} operator \cite{volterra2005theory}, which for $1$-dimensional signals $f(x)$ is given by $Vf(x) = \sum_{k=0}^{\infty}H_kf(x)$, where
\begin{equation}
H_kf(x)= 
\underbrace{\int \mkern -5mu \cdots \mkern -5mu \int}_{k\,\, \mathrm{times}}h^{(k)}(\tau_1, \ldots, \tau_k)f(x - \tau_1)\cdots f(x - \tau_k) \,d\tau_1\cdots d\tau_k
\label{eq:4}
\end{equation}

and $H_0$ is some constant. Note how Definition \ref{def:3.3} reduces to (\ref{eq:4}) in the case of scalar-valued feature maps on Euclidean space. The higher order terms ($k > 1$) can model nonlinear interactions within a given receptive field \cite{Banerjee-TPAMI20,zoumpourlis2017non}. The problem with directly applying equation (\ref{eq:4}) is that higher order cross-correlations quickly become computationally intractable as we increase the dimension of the input signals and the degree of the expansion. Moreover, the Volterra kernels $h^{(k)}$ contain many redundancies in the typical case where we assume $h^{(k)}$ to be symmetric, i.e. $h^{(k)}(\tau_1,\ldots,\tau_k) = h^{(k)}(\tau_{\sigma(1)}, \ldots, \tau_{\sigma(k)})$ for any permutation $\sigma$.

Kernel methods offer a solution to these issues. In \cite{Franz-Scholkopf}, it was shown that the estimation of $V^{(m)}f(x) = \sum_{k=0}^m H_kf(x)$ is equivalent to optimizing against the hypothesis
$\widehat{V}^{(m)}f(x) = \sum_{j=1}^N \gamma_jk^{(m)}(f(x), f(x_j))$, where $N$ is the number of training samples, $\gamma_j$ are learnable coefficients, $x_j$ is the $j^{\mathrm{th}}$ training sample, and $k^{(m)}$ is a reproducing kernel of the form $k^{(m)}(f(x_1), f(x_2)) = (c + f(x_1)^{\top}f(x_2))^m$, for some constant $c$. Since we would like $f(x)$ to be representative of an entire dMRI volume, we hierarchically construct the features $f(x_j)$ using the CKN formulation given in \ref{section:2.2} with polynomial kernels $\kappa_{\ell}$. This yields an end-to-end implementation of an (approximate) $m^{\mathrm{th}}$ order Volterra expansion for arbitrary $m$ that {\it relies on the same number of learnable parameters as would a first order CNN}, since the only modification we make is to the degree of $\kappa_{\ell}$.
\vspace{2pt}

\section{Network Operations}\label{arch}

In this section, we present the details of the GEVNet architecture. The CKN architecture is described in detail in \cite{Mairal2016}, so we omit it as it is not the novelty of this work. 

\begin{figure}[h]
\centering
\includegraphics[width=0.75\linewidth]{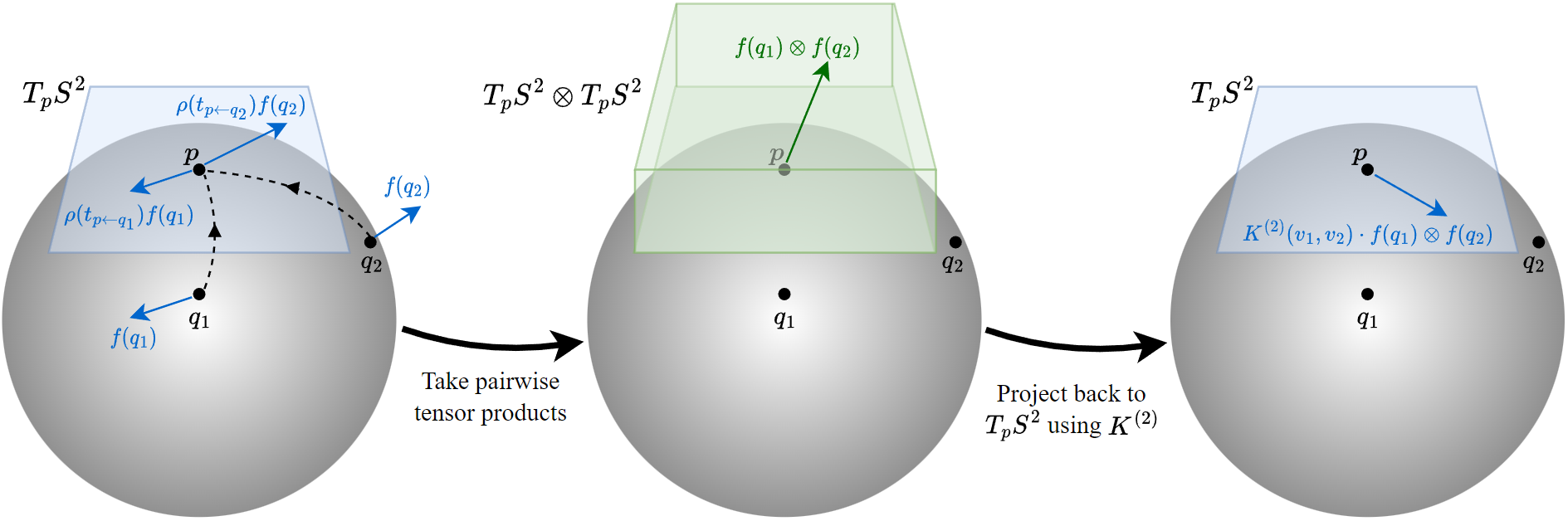}
\caption{A high level illustration of the second order gauge equivariant convolution on vector fields. We abbreviate $\rho(t_{p \gets q_1})f(q_1) \otimes \rho(t_{p \gets q_2})f(q_2)$ as $f(q_1) \otimes f(q_2)$.}
\label{gevconv_fig}
\end{figure}

{\bf Convolution:} Just as in \cite{PimDeHaan2020, WeilerCesa}, we view a feature map $f$ as being decomposed by the irreducible representations of its tensor type $\rho$. That is, since $G = \mathsf{SO}(2)$, we think of a feature $f(p)$ as transforming via a representation of the form $\rho_0 \oplus \rho_1 \oplus \cdots \oplus \rho_{\ell}$ for some bandlimit $\ell$, where $\rho_i$ is the $i^{\mathrm{th}}$ Fourier mode. The GEVNet computes second order gauge equivariant convolutions on the sphere $S^2$ with $\ell= 1$, meaning our hidden layers will have feature maps that transform according to $\rho_0 \oplus \rho_1$, i.e. we associate a scalar and a tangent vector to each $p \in S^2$, where of course we allow for multiple such pairs determined by the number of channels. To do this, we make the simplifying assumption that the pairwise tensor products $f(v_1) \otimes f(v_2)$ for all $v_1, v_2 \in T_pS^2$ only occur across tensor features of the same frequency (we do not distribute tensor products over direct sums). Explicitly, if $f(v_1) = s_1 \oplus r_1$ and $f(v_2) = s_2 \oplus r_2$, where $s_i$ is a scalar and $r_i$ is a tangent vector, then we stipulate that $\boxed{f(v_1) \otimes f(v_2):= (s_1 \otimes s_2) \oplus (r_1 \otimes r_2)}$. Under this assumption, a second order kernel $K^{(2)}(v_1, v_2)$ of type $(\rho_0 \oplus \rho_1, \rho_0 \oplus \rho_1)$ will take the form of (\ref{eqn:7}),
\begin{wrapfigure}{l}{.5\textwidth}
\begin{equation}
\begin{pmatrix}
    K_{00}^{(2)}(v_1, v_2)^{{\color{red} 1 \times 1}} & K_{01}^{(2)}(v_1, v_2)^{{\color{red} 1 \times 4}} \\
    K_{10}^{(2)}(v_1, v_2)^{{\color{red} 2 \times 1}} & K_{11}^{(2)}(v_1, v_2)^{{\color{red} 2 \times 4}}
\end{pmatrix}
\label{eqn:7}
\end{equation}
\end{wrapfigure}
where $K^{(2)}_{ij}$ is the learned linear combination of second order basis kernels of type $(\rho_i, \rho_j)$, generated using \ref{prop:1}. Each matrix block in (\ref{eqn:7}) is annotated with its size in red. We then contract $K^{(2)}(v_1, v_2)$ with the input feature map's pairwise interaction $f(v_1) \otimes f(v_2)$ using a PyTorch Einstein summation operation, and we add this to the result of passing $f$ through a first order gauge equivariant convolution. Figure \ref{gevconv_fig} illustrates the mechanics of a second order convolution on a pair of neighboring feature vectors. Our implementation of the convolution is largely inspired by \cite{PimDeHaan2020}. In particular, we modify the precomputed steerable kernel using a quadrature interpolation scheme to homogenize image patches, since not all points on a spherical grid will possess the same number of adjacent neighbors.

{\bf Nonlinearity:} We use a regular nonlinearity as described in \cite{dehaan2021}. Once again treating each $f(p)$ as the coefficients of Fourier modes of a periodic function with bandlimit $\ell = 1$, we perform an inverse Fourier transform to yield $N$ spatial samples at each $p \in S^2$. We then apply ReLU/BatchNorm operations in the spatial domain, before returning to a new feature via the Fourier transform.

{\bf Pooling:} Consider a high resolution grid $\mathcal{G}_{h}$ and a low resolution grid $\mathcal{G}_{l}$. The pooled feature $f_{\mathrm{avg}}(p)$ at a point $p \in \mathcal{G}_{l}$ is obtained by (a) considering the embedding of $p$ in $\mathcal{G}_{h}$, (b) computing $p$'s high resolution neighbors $q_i \in \mathcal{G}_{h}$, (c) parallel transporting and reorienting the input features $f(q_i)$ to the frame at $p$, and (d) averaging the transported features. In practice, the parallel transport terms involved in the pooling are precomputed during model initialization. 


\smallskip
\vspace{-10pt}
\section{Experiments}\label{expts}

In this section we present several experiments, beginning with computer vision benchmarks on $S^2$ and ending with an application in neurology. We use the following notation to describe network architectures: $\operatorname{\mathbf{GEVConv}}(c_{\mathrm{in}}^{\rho_{\mathrm{in}}}, c_{\mathrm{out}}^{\rho_{\mathrm{out}}})$ denotes a second order gauge equivariant convolution layer taking in $c_{\mathrm{in}}$  feature maps of type $\rho_{\mathrm{in}}$ and outputting $c_{\mathrm{out}}$  feature maps of type $\rho_{\mathrm{out}}$, while $\operatorname{\mathbf{GEConv}}(c_{\mathrm{in}}^{\rho_{\mathrm{in}}}, c_{\mathrm{out}}^{\rho_{\mathrm{out}}})$ is the analogous first order gauge equivariant layer. We assume that every convolution (except the last) is followed by a regular nonlinearity, so this is tacitly implied in the notation. If we require a pooling layer after the nonlinearity, then we embellish the above with a $\downarrow$. \textit{We remind the reader that the purpose of these experiments is to demonstrate the parameter efficiency of $\operatorname{\mathbf{GEVConv}}$ over $\operatorname{\mathbf{GEConv}}$ and other first order convolution variants, due to the additional expressivity provided by higher order convolution terms.} 

\subsection{Spherical MNIST}
\label{mnist_section}

\begin{wrapfigure}{r}{0.38\textwidth}
\vspace{-20pt}
\centering
\includegraphics[width=0.35\textwidth]{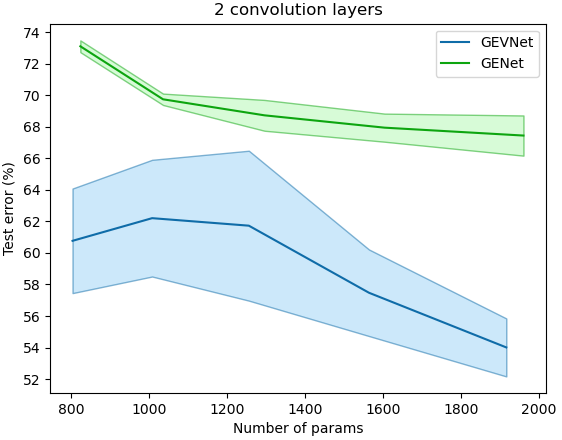}
\caption{Ablation on GEVNet versus GENet parameter counts.}
\label{mnist_chart}
\end{wrapfigure}

The spherical MNIST data were generated using the code released in \cite{Cohen-ICLR18}, except we project onto an icospherical grid instead of a Driscoll-Healy grid. Nevertheless, the original input resolutions are comparable. There are three variations of the dataset, namely NR/NR, NR/R, and R/R, where for example NR/R means that the training data has not been rotated post-projection while the test data has been randomly rotated.

To begin, we compare the test performance of a second order GEVNet versus that of a first order GENet on NR/NR \textit{while varying the number of learnable parameters in each model}. We consider three cases: 2-, 3- and 4-layer networks (see appendix for 3 and 4 layers), where for each case we initialize five GEVNet models and five GENet models with progressively increasing parameter counts to compare the trends in classification error. For each choice of GEVNet architecture, we ensure that there is a corresponding GENet with at least as many learnable parameters. For instance, to a 2-layer GEVNet (804 params) with architecture $\operatorname{\mathbf{GEVConv}}(1^{\rho_0}, 2^{\rho_0 \oplus \rho_1})^{\downarrow} \rightarrow \operatorname{\mathbf{GEVConv}}(2^{\rho_0 \oplus \rho_1}, 2^{\rho_0})$ we associate a 2-layer GENet (824 params) with architecture $\operatorname{\mathbf{GEConv}}(1^{\rho_0}, 2^{\rho_0 \oplus \rho_1})^{\downarrow} \rightarrow \operatorname{\mathbf{GEConv}}(2^{\rho_0 \oplus \rho_1}, 4^{\rho_0})$. Such a constraint can be seen as adversarial against GEVNets. For a given number of layers, all models are trained using the same optimization hyperparameters and for the same duration (20 epochs). The plot in Figure \ref{mnist_chart} reports mean test errors across five runs per model and shaded regions convey the standard error of the mean (95\% confidence). Note how in all cases the GEVNet's test error is consistently lower than the GENet's as we vary parameter counts. This suggests that the GEVNet produces a richer class of features in comparision to a GENet for the same number of learnable parameters.

Finally, in Table \ref{mnist_benchmark} we report our best GEVNet model on the standard spherical MNIST benchmarks in comparison to previous works. Our model achieves comparable performances for a far fewer number of parameters than other spherical CNNs.

\begin{table}[t]
        \begin{minipage}{.55\linewidth}
    
                \begin{center}
                \normalsize
    
                \caption{Test accuracies (\%) for spherical MNIST classification tasks.}
                \footnotesize
    
                    \begin{tabular}{lcccc}
        \label{mnist_benchmark}
          Method                                         & NR/NR             & NR/R             & R/R       & Params\\
        \hline
          Cohen et al. \cite{Cohen-ICLR18}               & 95.59               & 93.40           & 94.62    & 58k   \\
          Kondor et al. \cite{CGNet}                     & 96.40               & 96.00           & 96.60    & 256k  \\
          Esteves et al. \cite{esteves2020}              & \textbf{99.37}               & 99.08           & 99.37    & 58k   \\
          Banerjee et al. \cite{Banerjee-TPAMI20}        & 96.72               & 96.10           & 96.71    & 46k   \\
          Cobb et al. \cite{cobb2021efficient}           & 99.35               & \textbf{99.34}           & \textbf{99.38}    & 58k   \\
          \hline
          GENet                                          & 97.30               & 95.86          & 95.99          & 45k    \\
          GEVNet                                         & 98.02               & 96.91           & 97.43          & \textbf{31k}    \\
        \hline
        \end{tabular}
                \end{center}
    
        \end{minipage}\hspace{.05\linewidth}%
        \begin{minipage}{.4\linewidth}
            \vspace{-.325cm}
                \begin{center}
                \normalsize
    
                \caption{Test root mean squared errors for QM7 regression task.}
                \footnotesize
    
                    \begin{tabular}{lcc}
                    \label{qm7_benchmark}
            Method                  & RMSE             & Params \\
            \hline
            Cohen et al. \cite{Cohen-ICLR18}            & 8.47             & 1.4M \\
            Kondor et al. \cite{CGNet}          & 7.97             & 1.1M \\
            Banerjee et al.  \cite{Banerjee-TPAMI20}       & 5.92             & 128k \\
            Cobb et al.  \cite{cobb2021efficient}           & \textbf{3.16}             & 337k \\
            \hline
            GEVNet                  & 5.57              & \textbf{70k}  \\
            \hline
        \end{tabular}
                \end{center}
    
        \end{minipage} 
    \end{table}

\subsection{Atomic Energy Prediction}

Here we consider the application of a GEVNet to the benchmark QM7 dataset \cite{blum2009970,rupp2012fast}, with the goal of regressing atomization energies of molecules given their consituent atoms' positions and charges. Our setup is identical to that of \cite{Cohen-ICLR18}, using a rotation and translation invariant Coulomb matrix representation. We present our result in Table \ref{qm7_benchmark}. We come in second in RMSE but at a significant reduction of parameters in comparison to all other models.

\subsection{Classification of Neurodegenerative Disorders using Brain dMRIs} 
\label{medical_expt}

{\bf Data Description:} The dMRI data pool we used consisted of brain scans from 85 patients with Lewy Body Disease (DLB), 112 patients with Alzheimer's Disease (AD) and 436 patients with Parkinson's Disease (PD). All the scans were first axis aligned, centered, eddy current corrected, and brain extracted using FSL and pnlNipype \cite{Andersson2016a, pnlNipype}. Since the data were pooled from different magnetic resonance (MR) scanners (Siemens 3T, GE 3T, and Philips 3T) possessing distinct acquisition parameters, the scans underwent a retrospective harmonization step to increase inter-scanner compatability. This was done using the publicly available dMRI harmonization software in \cite{BillahT2019, RathiPaper}. 
After harmonization, the scans were affinely registered to a common MNI (Montreal Neurological Institute) space and downsampled to a voxel size of 2 mm$^3$. 
Finally, the image intensities at a given voxel (to be thought of as a scalar field on $S^2$) were passed through a spherical (radial basis function) interpolation and re-sampled onto a Healpix grid \cite{deepsphere_iclr} with 192 grid points. This corresponds to 96 distinct magnetic field gradient directions due to antipodal symmetry. 

Since the amount of raw data is relatively small (especially for AD and DLB), we set up five (80/20 split) train/test folds for each of the three possible classification problems (AD v. DLB, AD v. PD, and DLB v. PD) to obtain averaged performance estimates. To mitigate class imbalance, we then augment each training set using a data synthesis technique called \textit{mixup} \cite{Zhang2018ICLR} such that there are 400 training samples per class.



{\bf Network Architecture:} An input dMRI volume can be seen as 4-dimensional, consisting of a 3D MRI for each diffusion-sensitized gradient direction (i.e. for each $b$-vector). A $b$-vector is the vector pointing in the direction of the applied diffusion-sensitized magnetic field gradient. Alternatively, one can think of a dMRI volume as a 3D lattice, where at each lattice point (voxel) we have a scalar field on $S^2$.

\begin{figure}[t]
\centering
\includegraphics[width=\linewidth]{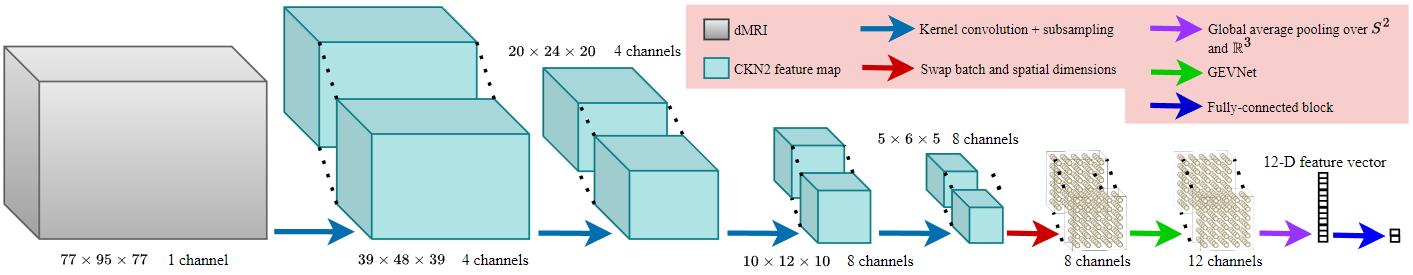}
\caption{A schematic of the CKN2 + GEVNet model. The GEVNet is of the form $\operatorname{\mathbf{GEVConv}}(8^{\rho_0}, 8^{\rho_0 \oplus \rho_1}) \rightarrow \operatorname{\mathbf{GEVConv}}(8^{\rho_0 \oplus \rho_1}, 12^{\rho_0 \oplus \rho_1}) \rightarrow \operatorname{\mathbf{GEVConv}}(12^{\rho_0 \oplus \rho_1}, 12^{\rho_0})$.}
\label{network}
\end{figure}

With this in mind, the data flow for a single dMRI volume is as follows: (1) each 3D volume (corresponding to a given $b$-vector) is passed as input to a CKN whose domain is $\mathbb{R}^3$, which summarizes diffusion phenomena occurring across voxels, (2) the resulting feature maps are concatenated along a fourth dimension, (3) each scalar field (on $S^2$) is passed to a spherical CNN, thus extracting a feature vector that is representative of the intra-voxel diffusion, (4) obtain a single feature vector by global average pooling, and (5) the final feature vector is classified by a fully connected block. Steps 1 and 2 are easily implemented in PyTorch by letting the $b$-vector dimension be the batch dimension. Then, we swap the batch and spatial dimensions before step 3 since the $b$-vector directions become the grid positions on $S^2$. Figure \ref{network} depicts a schematic of this network. Our CKN implementation is an adaptation of \cite{dexiong_code}.

{\bf Ablation Study:} We test the effect of higher order convolutions on three neurodegenerative disorder classification problems: AD v. DLB, AD v. PD, and DLB v. PD. Ablation is performed on the order of the inter-voxel convolution (CKN) by varying the degree of a polynomial kernel, and on the order of the intra-voxel convolution (spherical CNN) by exchanging a second order GEVNet with a first order DeepSphere (DS) \cite{deepsphere_iclr}. The abbreviations CKN1 and CKN2 refer to a CKN with linear and quadratic kernels, respectively. Although not identical, a CKN1 can be viewed as a close approximation to a standard CNN. We also include two additional ablated models: a lone CKN2 that ignores diffusion-sensitized gradient directions by flattening the feature map after the last layer, and a lone GEVNet that ignores diffusion phenomena across voxels.

Each CKN has a \textit{fixed} architecture consisting of a mere 3164 learnable parameters. Both the GEVNet and the DeepSphere consist of three convolution layers, with the GEVNet having 12424 parameters and the DeepSphere having 13696 parameters. All other optimization-related hyperparameters (e.g. learning rate, weight decay, number of epochs) are kept the same across classification tasks, and are provided in detail in the appendix. The test accuracies (averaged over five train/test folds) are presented in Table \ref{medical_table}.

\begin{table}[htbp]
  \centering
  \caption{Test accuracies (\%) on three neuroimaging classification problems.}
    \begin{tabular}{lccccc}
        \hline
          Architecture                & AD v. DLB             & AD v. PD             & DLB v. PD       \\
        \hline
          CKN2               & 86.76             & 90.32           & 98.06         \\
          GEVNet               & 77.08           & 72.36           & 81.12          \\
          \hline
          CKN1 + DS                         & 86.76             & 79.16         & 89.24           \\
          CKN1 + GEVNet                     & 90.34             & 86.30         & 93.30        \\
          CKN2 + DS                         & 89.82           & 95.07         & \textbf{98.46}           \\
          CKN2 + GEVNet                     & \textbf{92.86}           & \textbf{98.36}         & 98.27          \\
        \hline
    \end{tabular}
  \label{medical_table}
\end{table}

For a fixed choice of CKN, we find that an accompanying GEVNet outperforms its DS counterpart by a significant margin, except in the case of CKN2 + GEVNet versus CKN2 + DS on DLB v. PD, where the accuracies are highly comparable. This is significant, given that the GEVNet is at an over 1k parameter disadvantage against the DeepSphere. We attribute the exceptional case to the observation that a lone CKN2 is already enough to discriminate DLB v. PD with above 98\% accuracy, leaving little to be gained by a spherical component. Said differently, this suggests that DLB and PD are distinct enough to be separated by macroscale volumetric features, as opposed to microscale diffusion features. Conversely, for a fixed choice of spherical CNN, we find that replacing the accompanying CKN's linear kernel with a quadratic kernel yields a boost in performance. Thus, we see that CKN2 + GEVNet is the overall best classifier, as it performs second order convolutions both across and within voxels.
\section{Conclusions and Outlook}
\label{conc}
In this paper, we generalized the first order gauge equivariant convolution presented in \cite{Cohen-Weiler-Kicanaoglu-Welling2019} to its higher order analogue. Our theoretical contribution applies to arbitrary Riemannian manifolds and to arbitrary tensor fields on such manifolds. Our resulting implementation on $S^2$, the GEVNet, was applied to two computer vision benchmark datasets and to a neuroimaging classification problem as part of a compound architecture (CKN + GEVNet). These experiments firmly indicate that the GEVNet exhibits greater representational capacity than other first order GDL models, without sacrificing the key property of being equivariant to a symmetry group admitted by the base manifold. From these observations we can glean the importance of considering higher order convolutions (either as kernel convolutions or cross-correlations) and their parameter efficiency, especially in settings exhibiting spatially extended nonlinear interactions (e.g. diffusion of water molecules in \ref{medical_expt}) which are difficult to capture by solely relying on pointwise nonlinearities.

 From a stability perspective, it appears that the price we pay to obtain more expressive features with the GEVNet is an increase in test error variance (see the ablation study in \ref{mnist_section}). 
 Furthermore, we are inclined to believe that comparing first and second order networks sheds light on the underlying data's nature, where an inability to improve upon first order performance using higher order convolutions is indicative of a system that does not exhibit spatially extended nonlinear interactions. This is akin to using quadratic basis functions to interpolate samples from a system that is actually linear, i.e. there is little to be gained. Investigating and taming the cause of the increased variance, along with pinpointing the relation between the order of the convolutions and the system being modeled, will be the focus of our future work.

\section*{Disclosure of Funding}
This research was in part funded by the NIH NINDS and NIA via RF1NS121099 to Vemuri.

\medskip

{\small
\bibliographystyle{plainurl}
\bibliography{main}
}

\pagebreak

\renewcommand{\thesection}{\Alph{section}}




\setcounter{section}{0}
\section{Theory}

We use the same notation as described in the body of the paper, but we reiterate it here for convenience. $V^m_{\mathcal{K}}(f_{\mathrm{in}})$ denotes the $m^{\mathrm{th}}$ order expansion of a feature map $f_{\mathrm{in}}$, whose definition entails convolutions of $f_{\mathrm{in}}$ with a collection of $\mathsf{SO}(2)$-steerable kernels $\mathcal{K} = \{K^{(k)}: k = 1, \ldots, m\}$, where $k$ indicates the order of the kernel. $\rho_{\mathrm{in}}$ is a group homomorphism $\mathsf{SO}(2) \longrightarrow \mathsf{GL}(d_{\mathrm{in}}, \mathbb{R})$ that encodes the transformation law of the feature map $f_{\mathrm{in}}$'s coordinatization under $\mathsf{SO}(2)$-valued gauge transformations, where $d_{\mathrm{in}}$ is the dimension of the tensor spaces $\mathfrak{T}_{\mathrm{in}}M\vert_p$ in which individual features $f_{\mathrm{in}}(p)$ lie. Finally, we let $q_v := \operatorname{exp}_p(w_pv)$ for a gauge $w_p: \mathbb{R}^2 \longrightarrow T_pM$, and $t_{p \gets q_v}$ denotes the $\mathsf{SO}(2)$-valued gauge transformation taking the parallel transported frame at $q_v$ to the frame at $p$.

\subsection{Proof of Theorem 3.4}
\begin{manualtheorem}{3.4}
$V^m_{\mathcal{K}}$ is gauge equivariant. That is, if $w_p$ and $\widetilde{w}_p$ are two pointwise gauges at $p \in M$ related by an $\mathsf{SO}(2)$-valued gauge transformation $\widetilde{w}_p = w_p \circ t$, then $V^m_{\mathcal{K}}(f_{\mathrm{in}})(p)$ transforms as $\rho_{\mathrm{out}}(t^{-1})V^m_{\mathcal{K}}(f_{\mathrm{in}})(p)$.
\end{manualtheorem}

\begin{proof}
By linearity, it suffices to show that $K^{(k)} \star (\cdot)$ is gauge equivariant for each $k = 1, \ldots, m$. Let us embellish any coordinate dependent object with a $\sim$ when written w.r.t. $\widetilde{w}_p$. Assuming an initial coordinatization w.r.t. $w_p$, we have that $(\widetilde{K}^{(k)}\star \widetilde{f}_{\mathrm{in}})(p)$
\begin{align*}
 &=\int\limits_{\mathbb{R}^2}\mkern -7mu\cdots\mkern -7mu\int\limits_{\mathbb{R}^2}\widetilde{K}^{(k)}(\widetilde{v}_1, \ldots, \widetilde{v}_k)
 \left(\bigotimes_{i=1}^k\rho_{\mathrm{in}}(t_{p \gets q_{\widetilde{v}_i}})\widetilde{f}_{\mathrm{in}}(q_{\widetilde{v}_i})\right)d\widetilde{v}_1\cdots d\widetilde{v}_k \\
&=\int\limits_{\mathbb{R}^2}\mkern -7mu\cdots\mkern -7mu\int\limits_{\mathbb{R}^2}\widetilde{K}^{(k)}(t^{-1}v_1, \ldots, t^{-1}v_k)
 \left(\bigotimes_{i=1}^k\rho_{\mathrm{in}}(t_{p \gets q_{t^{-1}v_i}})\widetilde{f}_{\mathrm{in}}(q_{t^{-1}v_i})\right)d\widetilde{v}_1\cdots d\widetilde{v}_k \\
 &=\int\limits_{\mathbb{R}^2}\mkern -7mu\cdots\mkern -7mu\int\limits_{\mathbb{R}^2}\rho_{\mathrm{out}}(t^{-1})K^{(k)}(v_1, \ldots, v_k) 
\rho_{\mathrm{in}}^{\otimes k}(t)\rho_{\mathrm{in}}^{\otimes k}(t^{-1})
 \left(\bigotimes_{i=1}^k\rho_{\mathrm{in}}(t_{p \gets q_{v_i}})f_{\mathrm{in}}(q_{v_i})\right)dv_1\cdots dv_k \\
 &= \rho_{\mathrm{out}}(t^{-1})\int\limits_{\mathbb{R}^2}\mkern -7mu\cdots\mkern -7mu\int\limits_{\mathbb{R}^2}K^{(k)}(v_1, \ldots, v_k) 
 \left(\bigotimes_{i=1}^k\rho_{\mathrm{in}}(t_{p \gets q_{v_i}})f_{\mathrm{in}}(q_{v_i})\right)dv_1\cdots dv_k\\
 &= \rho_{\mathrm{out}}(t^{-1})(K^{(k)}\star f_{\mathrm{in}})(p).
\end{align*}
Note that no modification needs to be made to the integration measure under a gauge transformation since $\operatorname{det} t = 1$ for all $t \in \mathsf{SO}(2)$.
\end{proof}

\subsection{Proof of Proposition 3.5}
\begin{manualprop}{3.5}
\label{prop:3.5}
Let $K^{(1)}_{01}$ be a first order $\mathsf{SO}(2)$-steerable kernel of type $(\rho_0, \rho_1)$ and $K^{(1)}_{11}$ a first order $\mathsf{SO}(2)$-steerable kernel of type $(\rho_1, \rho_1)$, where $\rho_0$ is the trivial representation (i.e. $\rho_0(t) = 1$) and $\rho_1$ is the standard representation mapping $t\in\mathsf{SO}(2)$ to the usual $2 \times 2$ rotation matrix. Then, a second order kernel of the form $K^{(2)}_{11}(v_1, v_2):= K^{(1)}_{11}(v_1) \otimes K^{(1)}_{01}(v_2)$ is $\mathsf{SO}(2)$-steerable of type $(\rho_1, \rho_1)$.
\end{manualprop}

\begin{proof}
Recall the mixed-product property: given matrices $\mathbf{A}$, $\mathbf{B}$, $\mathbf{C}$, and $\mathbf{D}$ (of compatible sizes), we have that $(\mathbf{A} \otimes \mathbf{B})(\mathbf{C} \otimes \mathbf{D}) = \mathbf{AC} \otimes \mathbf{BD}$. Letting $t \in \mathsf{SO}(2)$ and letting MP denote the mixed-product property, we have that 
\begin{align*}
K_{11}^{(2)}(t^{-1}v_1, t^{-1}v_2) &= K^{(1)}_{11}(t^{-1}v_1)\otimes K^{(1)}_{01}(t^{-1}v_2) \\
&=\rho_1(t^{-1})K^{(1)}_{11}(v_1)\rho_1(t) \otimes \rho_0(t^{-1})K^{(1)}_{01}(v_2)\rho_1(t)\tag{by steerability}\\
&=(\rho_1(t^{-1})K^{(1)}_{11}(v_1) \otimes K^{(1)}_{01}(v_2))(\rho_1(t) \otimes \rho_1(t)) \tag{by MP}\\
&=\rho_1(t^{-1})\cdot K^{(1)}_{11}(v_1) \otimes K^{(1)}_{01}(v_2) \cdot \rho_1(t) \otimes \rho_1(t) \tag{by MP}\\
&= \rho_1(t^{-1})K_{11}^{(2)}(v_1, v_2)\rho_1^{\otimes 2}(t).
\end{align*}
\end{proof}

\vspace{-13pt}
Proposition \ref{prop:3.5} allows us to generate second order $\mathsf{SO}(2)$-steerable basis kernels from the first order solutions obtained in \cite{WeilerCesa}. For example, we can construct the (angular) second order basis kernel of type $(\rho_1, \rho_1)$ given by

\[
K^{(2)}_{11}(\theta_1, \theta_2) = K^{(1)}_{11}(\theta_1) \otimes K^{(1)}_{01}(\theta_2) = \begin{pmatrix}\operatorname{cos}2\theta_1 & \operatorname{sin}2\theta_1 \\ \operatorname{sin}2\theta_1 & -\operatorname{cos}2\theta_1\end{pmatrix} \otimes \begin{pmatrix} \operatorname{cos}\theta_2 & \operatorname{sin}\theta_2\end{pmatrix}.
\]

\section{Experiment Details and Addenda}
Our code is currently available as a downloadable ZIP file at \url{https://drive.google.com/file/d/1nELRtKKDqAbTyUD82kf2y6NIRexpYJP_/view?usp=share_link}, but will be streamlined as a public repository in the near future.
\subsection{Spherical MNIST}
We begin by presenting additional plots depicting our ablation study on parameter counts (Figure \ref{ablation_fig}). Note that the plots presented here depict the statistics (mean error and standard error of the mean) obtained over 10 runs per model. Again, we find that the GEVNet exhibits greater representational capacity than its first order analogue in all cases. All models for the ablation study were trained for 20 epochs, with an exponentially decaying initial learning rate of $3\mathrm{e}{-}4$. All regular nonlinearities mapped the hidden feature vectors to $N = 101$ spatial samples, and the precomputed basis kernels were homogenized using a single ring consisting of 1000 quadrature points. We used a batch size of 512, cross entropy loss, and an ADAM optimizer.


\renewcommand{\thefigure}{5}
\begin{figure}
\begin{subfigure}{.5\textwidth}
\centering
\includegraphics[width=0.85\linewidth]{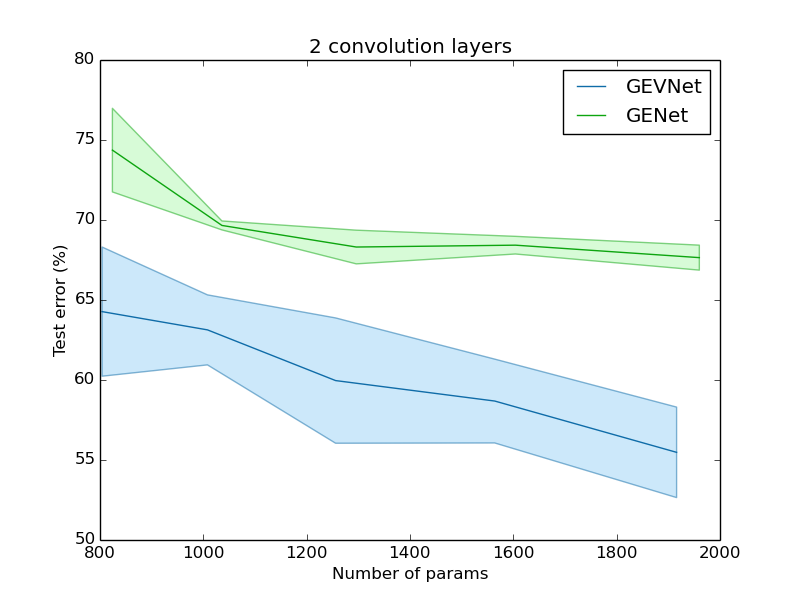}
\label{fig:sub1}
\end{subfigure}%
\begin{subfigure}{.5\textwidth}
\centering
\includegraphics[width=0.85\linewidth]{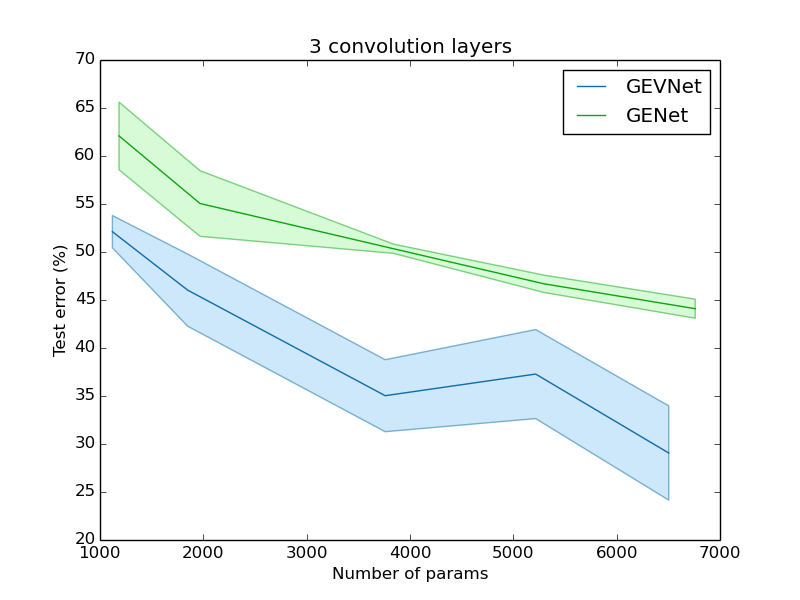}
\label{fig:sub2}
\end{subfigure}\\[1ex]
\begin{subfigure}{\textwidth}
\centering
\includegraphics[width=0.43\linewidth]{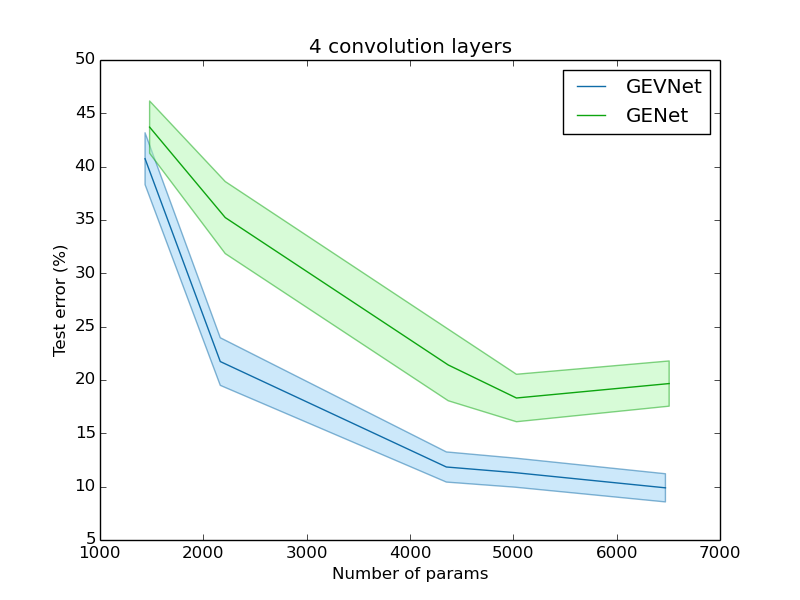}
\label{fig:sub3}
\end{subfigure}
\caption{Ablation on GEVNet versus GENet parameter counts for 2-, 3-, and 4-layer networks. Solid lines depict the mean test error over 10 runs per model and shaded regions depict the standard error of the mean.}
\label{ablation_fig}
\end{figure}

Next, we present the details of the reported GEVNet and GENet architectures on the MNIST benchmark tasks. The GEVNet's architecture (31k params) is given by
\begin{align*}
&\,\,\,\,\,\,\,\,\operatorname{\mathbf{GEVConv}}(1^{\rho_0}, 3^{\rho_0 \oplus \rho_1}) \rightarrow \operatorname{\mathbf{GEVConv}}(3^{\rho_0 \oplus \rho_1}, 3^{\rho_0 \oplus \rho_1})^{\downarrow} \\ &\rightarrow \operatorname{\mathbf{GEVConv}}(3^{\rho_0 \oplus \rho_1}, 8^{\rho_0 \oplus \rho_1})  \rightarrow \operatorname{\mathbf{GEVConv}}(8^{\rho_0 \oplus \rho_1}, 8^{\rho_0 \oplus \rho_1})^{\downarrow} \\ &\rightarrow \operatorname{\mathbf{GEVConv}}(8^{\rho_0 \oplus \rho_1}, 12^{\rho_0 \oplus \rho_1}) \rightarrow \operatorname{\mathbf{GEVConv}}(12^{\rho_0 \oplus \rho_1}, 12^{\rho_0 \oplus \rho_1})^{\downarrow} \\ &\rightarrow \operatorname{\mathbf{GEVConv}}(12^{\rho_0 \oplus \rho_1}, 12^{\rho_0}),
\end{align*}
while the GENet's architecture (45k params) is given by
\begin{align*}
&\,\,\,\,\,\,\,\,\operatorname{\mathbf{GEConv}}(1^{\rho_0}, 10^{\rho_0 \oplus \rho_1}) \rightarrow \operatorname{\mathbf{GEConv}}(10^{\rho_0 \oplus \rho_1}, 10^{\rho_0 \oplus \rho_1})^{\downarrow} \\ &\rightarrow \operatorname{\mathbf{GEConv}}(10^{\rho_0 \oplus \rho_1}, 16^{\rho_0 \oplus \rho_1})  \rightarrow \operatorname{\mathbf{GEConv}}(16^{\rho_0 \oplus \rho_1}, 16^{\rho_0 \oplus \rho_1})^{\downarrow} \\ &\rightarrow \operatorname{\mathbf{GEConv}}(16^{\rho_0 \oplus \rho_1}, 32^{\rho_0 \oplus \rho_1}) \rightarrow \operatorname{\mathbf{GEConv}}(32^{\rho_0 \oplus \rho_1}, 32^{\rho_0 \oplus \rho_1})^{\downarrow} \\ &\rightarrow \operatorname{\mathbf{GEConv}}(32^{\rho_0 \oplus \rho_1}, 32^{\rho_0}).
\end{align*}
Both models were trained for 30 epochs, with an exponentially decaying initial learning rate of $1\mathrm{e}{-}4$. Regular nonlinearities mapped the hidden feature vectors to $N = 51$ spatial samples, and the precomputed basis kernels were homogenized using a single ring consisting of 1000 quadrature points. We used a batch size of 256, cross entropy loss, and an ADAM optimizer.

All MNIST experiments were performed on a single NVIDIA A100 GPU. 

\subsection{Atomic Energy Prediction}
As described in \cite{Cohen-ICLR18}, this task is handled by an architecture consisting of a spherical component followed by an MLP. We do not modify the original MLP in any way, only substituting Cohen's $S^2$CNN with our GEVNet. The GEVNet's architecture is given by
\[
\operatorname{\mathbf{GEVConv}}(5^{\rho_0}, 5^{\rho_0 \oplus \rho_1})^{\downarrow} \rightarrow \operatorname{\mathbf{GEVConv}}(5^{\rho_0 \oplus \rho_1}, 10^{\rho_0 \oplus \rho_1})^{\downarrow} \rightarrow \operatorname{\mathbf{GEVConv}}(10^{\rho_0 \oplus \rho_1}, 10^{\rho_0}).
\]

The model was trained for 30 epochs with a learning rate of $1\mathrm{e}{-}3$ and a weight decay of $1\mathrm{e}{-}2$ using a batch size of 16, MSE loss, and an ADAM optimizer. All regular non-linearities mapped the hidden feature vectors to $N = 101$ spatial samples, and the precomputed basis kernels were homogenized using a single ring consisting of 1000 quadrature points. This experiment was performed on a single NVIDIA A100 GPU.

\subsection{Experimental Details for Classification of Neuro-degenerative Disorders}

All compound models depicted in Table 3 of the main paper were trained for 15 epochs, with an exponentially decaying initial learning rate of $1\mathrm{e}{-}5$ and a weight decay of $1\mathrm{e}{-}5$ using cross entropy loss and an ADAM optimizer. All regular non-linearities mapped the hidden feature vectors to $N = 7$ spatial samples, and the precomputed basis kernels were homogenized using a single ring consisting of 5000 quadrature points. Given the large input image size (approx. 563k voxels), we parallelize the training across 6 NVIDIA A100 GPUs while taking care to use a synchronized batch normalization. The reported DeepSphere's architecture was given by 
\[
\operatorname{\mathbf{DS}}(8, 32) \rightarrow \operatorname{\mathbf{DS}}(32, 32) \rightarrow
\operatorname{\mathbf{DS}}(32, 64).
\]


\end{document}